\documentclass{article}
\usepackage{ijcai16}

\usepackage{times}
\usepackage{marcstyle}

\title{AGM-Style Revision of Beliefs and Intentions from a Database Perspective\\(Preliminary Version)}
\author{Marc van Zee and Dragan Doder\\
Computer Science and Communications,\\ University of Luxembourg, University of Luxembourg}
\begin{document}

\maketitle

\begin{abstract}
We introduce a logic for temporal beliefs and intentions based on Shoham's database perspective. We separate strong beliefs from weak beliefs. Strong beliefs are independent from intentions, while weak beliefs are obtained by adding intentions to strong beliefs and everything that follows from that. We formalize coherence conditions on strong beliefs and intentions. We provide AGM-style postulates for the revision of strong beliefs and intentions. We show in a representation theorem that a revision operator satisfying our postulates can be represented by a pre-order on interpretations of the beliefs, together with a selection function for the intentions.
\end{abstract}

\section{Introduction}
Recently there has been an increase in articles studying the dynamics of intentions in logic~\cite{Ditmarsch2011,Icard2010,Lorini2008,vanderHoek2007,Lorini2009,grant2010}. Most of those papers take as a starting point the logical frameworks derived from Cohen and Levesque~\shortcite{Cohen1991}, which in turn formalize Bratman's~\shortcite{Bratman1987} planning theory of intention. In this paper, we take a different starting point, and study the revision of intentions from a \emph{database perspective}~\cite{Shoham2009}. The database perspective consists of a planner, a belief database and an intention database. Shoham~\shortcite{Shoham2016} describes it as ``(...) a generalization of the AGM scheme for belief revision, (...). In the AGM framework, the intelligent database is responsible for storing the planner's beliefs and ensuring their consistency. In the enriched framework, there are two databases, one for beliefs and one for intentions, which are responsible for maintaining not only their individual consistency but also their mutual consistency.'' (p.48) Shoham further developed these ideas with Jacob Banks, one of his PhD students, and behavioral economist Dan Ariely in the intelligent calendar application Timeful, which attracted over \$6.8 million in funding and was acquired by Google in 2015\footnote{http://venturebeat.com/2015/05/04/google-acquires-scheduling-app-timeful-and-plans-to-integrate-it-into-google-apps/}, who aim to integrate it into their Calendar applications. As Shoham
~\shortcite{Shoham2016} says himself: ``The point of the story is there is a direct link between the original journal paper and the ultimate success of the company.'' (p.47) Thus, it seems clear that his philosophical proposal has lead to some success on the practical side. In this paper, we investigate whether his proposal can lead to interesting theoretical insights as well. More specifically, the aim of this paper is to develop a suitable formal theory for the belief and intention database in the database perspective, and to study belief and intention revision for this theory. Following Shoham's proposal, our methodology is to generalize AGM revision~\cite{AGM1985} for temporal beliefs and intentions in terms of a representation theorem.

In the area of intention revision and reconsideration, Grant 
\etal{}~\shortcite{grant2010} combine intention revision with AGM-like postulates. There have also been a number of contributions applying AGM-style revision to action logics~\cite{Shapiro2011,Jin2004,Scherl2005,Scherl2003,Bonanno2007}. However, these proposals only characterize revision using a set of postulates, without proving representation theorems. There are also approaches that focus on the semantical level by postulating revision on a Kripke model~\cite{Baral2005}. Both \icard{}~\shortcite{Icard2010} and van Zee \etal{}~\shortcite{vanZee2015b} develop a logic based on Shoham's database perspective and prove representation theorems, but we will analyze the shortcomings of their proposals in the next section.

The three main technical results of this paper are as follows: First we review two recent formalisations based on Shoham's database perspective, namely the IPS framework~\cite{Icard2010} and the PAL framework~\cite{vanZee2015}, and we discuss the shortcomings of these logics. Secondly, we extend PAL in order to define a coherence condition on strong beliefs and intentions, and we separate strong, or intention-independent, beliefs from weak beliefs. Lastly, we characterize revision of beliefs and intentions through AGM-style postulates and we prove a representation theorem relating the postulates for revision to an ordering among interpretations and a selection function that accommodates new intentions while restoring coherence. The three next sections of the paper are in line with these three results.

\section{Preliminaries: The Database Perspective}

We review Shoham's database perspective, and we discuss the main limitations of two recent formalizations. For the first formalization (IPS), the separation between strong and weak beliefs is problematic, and the representation theorem only holds for a specific type of beliefs. For the second formalization (PAL), intention revision is missing.

\subsection{Shoham's Database Perspective}

Shoham's database perspective contains a planner (e.g., a STRIPS-like planner) that is itself engaged in some form of practical reasoning. In the course of planning, it may add actions to be taken at various times in the future to an intention database and add observations to a belief database. The intentions are \emph{future-directed intentions} of the form $(a,t)$, meaning that action $a$ will be executed at time $t$.\footnote{The notion of intention here is clearly quite restrictive and important characteristics of intentions are missing. See the conclusion for a discussion.} The beliefs are also time-indexed, and are of the form $p_t$, meaning that $p$ is true at time $t$. Shoham treats the planner as a ``black box'': It provides the databases with input but its internal workings are unknown. Shoham proposes informal revision procedures for beliefs and intentions based on the following coherence conditions:
\begin{enumerate}
\item 
If two intended actions immediately follow one another, the earlier cannot have postconditions that are inconsistent with the preconditions of the latter.
\item 
If you intend to take an action you cannot believe that its preconditions do not hold.
\item 
If you intend to take an action, you believe that its postconditions hold.
\end{enumerate}

Note that requirement 2 and 3 describe an asymmetry between pre-and postconditions: The postconditions are believed to be true after an intended action, but the preconditions may not. Therefore, we might think of the requirements as one of ``optimistic'' beliefs. According to Shoham~\shortcite{Shoham2009}: ``It is a good fit with how planners operate. Adopting an optimistic stance, they feel free to add intended actions so long as they are consistent with current beliefs.'' (p.7) 

\subsection{\icard{} (IPS)}
\icard{}~\shortcite{Icard2010} develop a ``formal semantical model to capture action, belief and intention, based on the `database perspective''' (p.1). They assume a set of atomic sentences \textsf{Prop} = $\{p,q,r,\ldots\}$ and deterministic primitive actions \textsf{Act}=$\{a,b,c,\ldots\}$. Entries in the belief database are represented by a language generated from:  $$\varphi := p_t\mid pre(a)_t\mid post(a)_t\mid Do(a)_t\mid\Box\varphi\mid\varphi\wedge\varphi\mid\neg\varphi$$ with $p\in\textsf{Prop}, a\in\textsf{Act}$, and $t\in\mathbb{Z}$. $p_t$ means that $p$ is true at time $t$, $Do(a)_t$ means that the agent does action $a$ at time $t$, and $pre(a)_t$ and $post(a)_t$ represent respectively the precondition and postcondition of action $a$ at time $t$.

\icard{} use a semantics of \emph{appropriate paths}. They define
$P=\mathcal{P}{(\mathsf{Prop}\cup\{pre(a),post(a) : a\in\mathsf{Act}\})},$ and a \emph{path} $\pi:\mathbb{Z}\rightarrow (P\times\mathsf{Act})$ as a mapping from a time point to a set of proposition-like formulas true at that time (denoted $\pi(t)_1$) and the next action $a$ on the path (denoted $\pi(t)_2$). They define an equivalence  relation $\pi\sim_t\pi'$, which means that $\pi$ and $\pi'$ represent the same situation up to $t$. Using this, they propose a notion of appropriateness:

\begin{definition}[Appropriate Set of Paths]
A set of paths $\Pi$ is \emph{appropriate} iff for all $\pi\in\Pi$:
\begin{itemize}
\item If $\pi(t)_2=a$, then $post(a)\in\pi(t+1)_1$,
\item If $pre(a)\in\pi(t)_1$, then there exists $\pi'\sim_t\pi$ s.t. $\pi'(t)_2=a$.
\end{itemize}
\end{definition}

The truth definition $\models_\Pi$ is defined relative to an appropriate set of paths $\Pi$, and the modality is defined as follows: $$\pi,t\models_\Pi \Box\varphi\text{, iff for all }\pi'\in \Pi\text{, if } \pi\sim_t\pi'\text{ then } \pi',t\models\varphi.$$ 
A model for a formula is an appropriate set of paths. They introduce an intention database $I=\{(a,t),\ldots\}$ as a set of action-time pairs $(a,t)$ and put the following coherence condition on their logic:
\begin{flalign*}
&Cohere^*(I) = \Diamond\bigwedge_{(a,t)\in I} pre(a)_t.
\end{flalign*}
This captures the intuition that an agent considers it possible to carry out all intended actions. They state that a set of models is coherent if and only if there exists a model in which $Cohere^*(I)$ is true. IPS distinguishes \emph{intention-contingent}, or weak, beliefs from \emph{non-contingent}, or strong, beliefs. Contingent beliefs $B^I$ are obtained from a belief-intention database $(B,I)$ as follows: $B^I=Cl(B\cup \{Do(a)_t:(a,t)\in I\}).$  In order to switch from belief bases to an appropriate set of paths, \icard{} introduce the functions $\rho$ and $\beta$: ``Given a set of formulas $B$, we can consider the set of paths on which all formulas of $B$ hold at time 0, denoted $\rho(B)$. Conversely, given a set of paths $\Pi$, we let $\beta(\Pi)$ be defined as the set of formulas valid at 0 in all paths in $\Pi$.'' (p.3)

The first issue with IPS is that the definition of non-contingent beliefs is problematic for coherence. The following example shows that non-contingent beliefs that are dependent on the actual path can lead to a coherent agent with inconsistent weak beliefs.

\begin{example}
Consider an IPS belief-intention base $(B,I)$ with the belief base $B=\{p_1,\neg\Diamond (do(a)_2\wedge p_1)\}$ and intention base $I=\{(a,2)\}$. While $B$ is consistent with $\Diamond pre(a)_2$, $B^I$ is inconsistent since $\neg\Diamond(do(a)_2\wedge p_1)\wedge do(a)_2$ derives $\neg p_1$, but this is inconsistent with the initial belief $p_1$. 
\end{example}

In Section 3 we will define our notion of coherence that together with a separation between strong and weak beliefs resolves this problem.

The second issue with IPS is that the definition of $\rho$ is circular, and as a result it does not seem to be possible for all formulas of their logic. Consider the following example.

\begin{example} Suppose $B=\{\Box p_1\vee \Box \neg p_1\}$. The set of paths $\rho(B)$ contains all paths $\pi$ for which $\pi\models B$, i.e. $\pi\models\Box p_1\vee \Box \neg p_1$. Take such a path $\pi$ arbitrary, so either $\pi\models\Box p_1$ or $\pi\models\Box \neg p_1$. Suppose $\pi\models\Box p_1$. Then it follows that for all $\pi'\sim_t\pi$ we have $\pi'\models p_1$. On the other hand, if $\pi\models\Box \neg p_1$, then it follows that for all $\pi'\sim_t\pi$ we have $\pi'\models \neg p_1$. So in both cases we end up with two different sets for $\rho(B)$. In other words, the set of paths $\rho(B)$ is not defined.
\end{example}

Essentially, according to IPS, $\rho(B) = \{\pi \mid \pi \models_\Pi B \}$. The set of appropriate paths for a belief base $B$ is thus constructed using $\models_\Pi$. However, $\models_\Pi$ is aleady defined relatively to some set of paths $\Pi$. Therefore, the definition of $\rho$ is circular. It seems that the $\rho$ function only works for belief bases containing no modalities (in other words, to construct a single path). We omit details for space constraints, but the construction of the canonical model in the proof of their representation theorem uses the $\rho$ function to switch from a belief base to a set of paths (see Proof Sketch in the Appendix of \icard{}~\shortcite{Icard2010}). Therefore, the representation theorem does not hold for all formulas of the logic, since it is not possible to apply the function $\rho$ to all beliefs. 

Summarizing, we recognize two shortcoming of the IPS framework as a formal basis for the database perspective: The definition of contingent beliefs is problematic, and the representation theorem does not hold for belief bases containing modalities.

\subsection{\vanzee{} (PAL)}
\vanzee{}~\shortcite{vanZee2015b} develop \logicname{} (\logic{}) as an alternative to IPS, claiming that the IPS framework contains an unsound axiom and that their logic is noncompact~\cite{vanZee2015}. They define $\act{}=\{a,b,c,\ldots\}$ as a finite set of deterministic primitive actions, and $\prop{}=\{p,q,r,\ldots\}\cup \{pre(a),post(a)\mid a\in\act{}\}$ as  a finite set of propositions. They denote atomic propositions with $\chi$. PAL differs syntactically from IPS in that the $\Box$-modality is indexed by a time-point. Their language \lang{} is inductively defined by the following BNF grammar:
\begin{align*}
\varphi ::= \chi_t\mid do(a)_t\mid\Box_t\varphi\mid\varphi\wedge\varphi\mid\neg\varphi
\end{align*}
\logic{} uses a CTL*-like tree semantics consisting of a tree $T=(S,R,v,act)$ where $S$ is a set of states, $R$ is an accessibility relation that is serial, linearly ordered in the past and connected, $\valp:S\rightarrow 2^{\prop{}}$ is a valuation function from states to sets of propositions, and $\vala:R\rightarrow \act{}$ is a function from accessibility relations to actions, such that actions are deterministic, i.e. if $\vala((s,s'))=\vala((s,s''))$, then $s'=s''$. Similarly to \icard{}, PAL uses an equivalence relation $\sim_t$ on paths. Using this equivalence relations they define a model as a pair $(T,\pi)$ on which the same conditions hold as an IPS model. PAL formulas are evaluated in a model on a path. \vanzee{} axiomatize \logic{} and show that it is sound and strongly complete, i.e. $T\vdash\varphi$ iff $T\models\varphi$. They characterize AGM belief revision in this logic and restrict time up to some $t$. Using these constraints, they are able to represent a belief set $B$ as a propositional formula $\psi$ such that $B=\{\varphi\mid\psi\vdash\varphi\}$ and they prove the Katsuno and Mendelzon (KM)~\shortcite{Katsuno1991} and the Darwiche and Pearl (DP)~\cite{Darwiche1997} representation theorems.

\vanzee{} only consider revision of PAL formulas and thus do not consider the problem of intention revision, and therefore also does not distinguish between strong and weak beliefs. However, PAL uses a standard branching time (CTL*-like~\cite{Reynolds2002}) semantics, and it therefore also does not suffer from the shortcomings that we addressed in the previous subsection. Therefore, we take PAL as a starting point and use the remainder of this paper to extend it so that we can study intention revision.

\section{The Belief-Intention Database}
We assume an intention database $I=\{(a,t),\ldots\}$ consisting of time-indexed actions. We next define a coherence condition on beliefs and intentions in PAL, and we separate strong and weak beliefs. 

\subsection{The Coherence Condition}

We first demonstrate that $Cohere^*(I)$ of IPS (Section 2.2) is too permissive because it allows models in which intentions are not jointly executable.

\begin{figure}[h!]
\centering
\begin{tikzpicture}[->,>=stealth',shorten >=1pt,auto,node distance=1cm,
  thin,
  main node/.style={circle,draw,fill=white,font=\sffamily\scriptsize\bfseries}, 
  every label/.style={font=\scriptsize},
  time node/.style={font=\scriptsize,text=black!90},
  test line/.style={gray!80, dashed,-}]
  
  \draw[test line] (-0.5,-1) -- (-0.5, 1.5);
  \draw[test line] (2,-1.0) -- (2, 1.5);
  \draw[test line] (4.6,-1.0) -- (4.6, 1.5);
  
  \node[time node] (t0) at (-0.5,1.6) {$t=0$};
  \node[time node] (t1) at (2.2,1.6) {$t=1$};
  \node[time node] (t2) at (4.6,1.6) {$t=2$};
  
  \node[main node,label={[label distance=-0.1cm]90:$\{pre(nmr)\}$}] (s0) at (-0.5,0) {$s_0$};
  \node[main node,label={[label distance=-0.1cm]90:$\{post(nmr)\}$},position=10:20mm from s0] (s1) {$s_1$};
   \node[main node,label={[label distance=-0.1cm]0:$\{post(nop)\}$},position=10:20mm from s1] (s2) {$s_2$};
  \node[main node, position=-10:20mm from s0] (s3) {$s_3$};
  \node[position=-130:0mm from s3,align=center,font=\scriptsize] (labels3) {$\{pre(ijcai),$\\$post(nop)\}$};  
  
  \node[main node,label={[label distance=-0.1cm]00:$\{post(ijcai)\}$},position=10:20mm from s3] (s4) {$s_4$};
  \node[main node,label={[label distance=-0.1cm]0:$\{post(nop)\}$},position=-10:20mm from s3] (s5) {$s_5$};
 
  \path[every node/.style={font=\sffamily\small}]
    (s0) edge[draw=black] node [above left] {nmr} (s1)
    (s1) edge[draw=black] node [above] {nop} (s2)
    (s0) edge[draw=black, line width=2, ultra thick] node [below left] {nop} (s3)
    (s3) edge[draw=black, line width=2, ultra thick] node [above left] {ijcai} (s4)
    (s3) edge[draw=black] node [below left] {nop} (s5);
\end{tikzpicture}
\caption{Example Model from $t=0$ to $t=2$.}
\label{fig:example1}
\label{fig:models}
\end{figure}
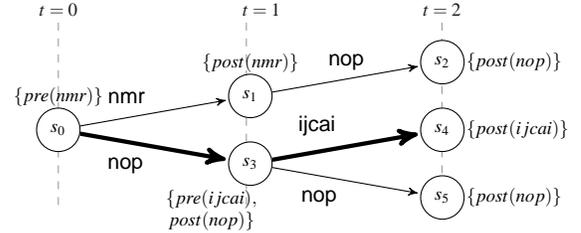

\begin{example}[Running Example]
\label{expl:1}
An agent is considering to attend the NMR workshop at time 0 and the IJCAI conference at time 1. Although it would like to attend both events, there is insufficient budget available. Consider a possible partial semantic model of this situation depicted in Figure~\ref{fig:example1}, where the thick path represents the actual path. In the actual path, the agent believes it does nothing at time 0 and attends IJCAI at time 1. It also considers it possible to attend the NMR workshop at time 0 and do nothing at time 1 in an alternative path. However, it does not consider it to be possible to attend both events. 
\end{example}

Suppose the agent of the running example has two intentions: $I=\{(nmr,0),(ijcai,1)\}$. Intuitively, the agent's intentions do not cohere with its beliefs, because it believes it cannot execute them both due to insufficient budget. However, according to $Cohere^*(I)$ the agent is coherent because the preconditions of all intentions hold on some path (namely the current path). Thus, $Cohere^*(I)$ does not fulfill Shoham's coherence condition 1 (Section 2.1). More specifically, the problem is that it is not possible to define the precondition of a set of actions in terms of preconditions of individual actions, because it cannot be ensured that all the intentions are fulfilled on the same path as well. Therefore, in order to formalize a coherence condition in PAL, we extend the language with preconditions of finite action sequences, which ensures that after executing the first action, the precondition for the remaining actions are still true. We modify the language, the definition of a model, the axiomatization, and we show that the new axiomatization is sound and strongly complete. We call the new logic \prelogic{} (\logicname{} with extended Preconditions).

\begin{definition}[\prelogic{} Language] The language \prelang{} is obtained from \lang{} by adding $\{pre(a,b,\dots)_t\mid \{a,b,\ldots\}\subseteq \act{},t\in\mathbb{N}\}$ to the set of propositions. Moreover, $Past(t)$ consists of boolean combinations of $p_{t'}, pre(a,b,\ldots)_{t'}, \Box_{t'}\varphi$ and $do(a)_{t'-1}$ where $t'\le t$ and $\varphi$ is some formula from \prelang{}.
\end{definition}

We also extend the definition of a model accordingly.

\begin{definition}[\prelogic{} Model]
A \emph{model} is a pair $(T,\pi)$ with $T=\tree{}$ such that for all $\pi\in T$ the following holds:
\begin{enumerate}
\item If $\vala(\pi_t)=a$, then $post(a)\in \valp(\pi_{t+1})$,
\item If $pre(a)\in v(\pi_t)$, then there is some $\pi'$ in $T$ with $\pi\sim_t\pi'$ and $\vala(\pi'_t)=a$,
\item If $pre(\ldots,a,b)_t\in v(\pi_t)$, then $pre(\ldots,a)_t\in v(\pi_t)$,
\item If $pre(a,b,\ldots)_t\in v(\pi_t)$, then there is some $\pi'$ in $T$ with $\pi\sim_t\pi'$, $act(\pi'_t)=a$, and $pre(b,\ldots)_{t+1}\in v(\pi'_{t+1})$.
\end{enumerate}

We refer to models of \prelogic{} with $m_1,m_2,\ldots$, we refer to sets of models with $M_1,M_2,\ldots$, and we refer to the set of all models with \mods{}. 
\end{definition}

\begin{definition} The logic \prelogic{} consists of the all the axiom schemas and rules of PAL~\cite{vanZee2015b} (Def. 7), and the following two:

$pre(\ldots,a,b)_t\rightarrow pre(\ldots,a)_t$\hfill (A11)

($pre(a,b,\ldots)_t\wedge do(a)_t)\rightarrow pre(b,\ldots)_{t+1}$\hfill(A12)\\

\noindent
The relation $\vdash$ is defined in the usual way with the restriction that necessitation can be applied to theorems only. 
\end{definition}

\begin{theorem}[Completeness Theorem] The logic \prelogic{} is sound and strongly complete, i.e. $T\vdash\varphi$ iff $T\models\varphi$.\footnote{We provide the full proofs of all theorems and propositions in this paper in a separate technical report: http://www.dropbox.com/s/798g6ki0aqt3gds/ijcai2016proofs.pdf?dl=1}
\end{theorem}

The proof of the theorem is a direct extension the proof of \vanzee{} with the new axioms and the new conditions on a model.  

Note that it is not directly possible in \prelogic{} to express preconditions for actions that do not occur directly after each other. In order to do so, we simply make a disjunction over all possible action combinations in the time points in between the actions. Thus, if for instance $\act{}=\{a,b\}$ and $I=\{(a,1),(b,3)\}$, then $Cohere(I) = \Diamond_0\bigvee_{x\in\act{}} pre(a,x,b)_1 =\Diamond_0( pre(a,a,b)_1\vee pre(a,b,b)_1)$.\footnote{Our construction of preconditions over action sequences may lead to a coherence condition involving a big disjunction. This is a drawback in terms of computational complexity. Alternatively, one may explicitly denote the time of each precondition, e.g. $pre(a,b)_{(t_1,t_2)}$. We chose the former since it is conceptually closer to the original syntax, but the latter can be implemented straightforwardly.}

\begin{definition}[Coherence]
Given an intention database $I=\{(b_{t_1},t_1),\ldots,(b_{t_n},t_n)\}$ with $t_1<\ldots<t_n$, let
\begin{flalign*}
Cohere(I) = \Diamond_0\bigvee_{\substack{a_k\in Act{} : k\not\in\{t_1,\ldots, t_n\}\\a_k=b_k : k\in\{t_1,\ldots,t_n\}}} pre(a_{t_1},a_{t_{1}+1},\ldots,a_{t_n})_{t_1}.
\end{flalign*}
For a given set of models $M$, we say that $(M,I)$ is \emph{coherent} iff there exists some $m\in M$ with $m\models Cohere(I)$.
\end{definition}

We now show that the new coherence conditions correctly specifies that the agent of our running example is not coherent.

\begin{example}[Continued] The models of the agent of our running example are not coherent with the intention database $I=\{(nmr,0),(ijcai,1)\}$, because the agent does not have the possibility to jointly execute both intentions (i.e. the preconditions for both actions together is false). Thus, none of the models satisfy $\Diamond_0 pre(nmr,ijcai)_0$, even though they satisfy $\Diamond_0 pre(nmr)_0\wedge \Diamond_0 pre(ijcai)_1$.
\end{example}

\subsection{Strong and Weak Beliefs}

The idea behind strong beliefs is that they represent the agent's ideas about what is inevitable, no matter how it would act in the world. Weak beliefs, on the other hand, represent the beliefs of an agent that are the consequence of its planned actions~\cite{vanderHoek2003}. Formally, we define strong beliefs at some time $t$ as formulas that start either with $\Diamond_t$ or $\Box_t$. The set of all \emph{strong beliefs} $\sbel{}_t$ in time $t$ for \lang{} is inductively defined by the following BNF grammar: $$\varphi ::= \ \Box_t\psi\mid \varphi\wedge\varphi\mid\neg\varphi,$$ where $\psi\in \lang{}$ and $t\in \timeflow{}$.  A \emph{strong belief set in $t$} is $B_t\subseteq\sbel_t$. In the remainder of this paper, we assume $t=0$ and we simply write $\sbel{}$ and $B$ to abbreviate $\sbel{}_0$ and $B_0$.

The weak beliefs $WB(B,I)$ are obtained from the strong beliefs by adding beliefs that are contingent on the intentions of the agent. In other words, the agent weakly believes everything that it strongly believed and moreover that all intentions will be realised, and everything that follows from this: $$WB(B,I)=Cl(B\cup \{do(a)_t\mid (a,t)\in I\}).$$

\begin{example}[Continued] Some example strong beliefs of the agent of Figure~\ref{fig:example1} are $\Diamond_0 do(nmr)_0$, $\Diamond_0 do(ijcai)_1$, and $\Box_0 \Diamond_1 post(nop)_2$. If the agent has the intention database $I=\{(ijcai,1)\}$, then $post(ijcai)_2$ is one of its weak beliefs.
\end{example}

\section{Belief and Intention Revision}
\paragraph{Postulates} 
Following KM, we fix a way of representing a belief set $B$ consisting of strong beliefs by a propositional formula $\psi$ such that $B=\{\varphi\mid\psi\vdash\varphi\}$. Since intentions and beliefs that have been added by a planner are naturally bounded up to some time point $t$, we define a bounded revision function and we restrict the syntax and semantics of \prelogic{} up to a specific time point. As a consequence, it is then possible to obtain the single formula $\psi$ for a set of strong beliefs $B$ (Corollary~\ref{cor:strongbeliefformula}). We first define some notation that we use in the rest of this paper.

\begin{definition}
An agent is a pair $(\psi,I)$ consisting of a belief formula $\psi$, and an intention base $I$. $\agents{}$ denotes the set of all agents, $\sbel$ denotes the set of all strong beliefs, $\ib{}$ denotes the set of all intentions, and $\idb$ denotes the set of all intention databases. We denote $\agents{},\sbel{},\ib{}$, and $\idb$ bounded up to $t$ with respectively $\agents\rest{}, \sbel{}\rest{}, \ib{}\rest{}$, and $\idb\rest{}$. However, if the restriction is clear from context, we may omit the superscript notation.
\end{definition}

We now define a bounded revision function $\agr$ revising an agent $(\psi,I)$ with a tuple $(\varphi,i)$ consisting of a strong belief $\varphi$ and an intention $i$, denoted $(\psi,I)\agr(\varphi,i)$, where $t$ is the maximal time point occurring in $\psi, I, \varphi$, and $i$. 

\begin{definition}[Agent Revision Function] 
An \emph{Agent revision function} $\agr:\agents{}\times (\sbel{}\times \ib{})\rightarrow \agents$ maps an agent, a strong belief formula, and an intention--- all bounded up to $t$--- to an agent bounded up to $t$ such that if, \\
$(\psi,I)\agr(\varphi,i) = (\psi',I')$,\\
$(\psi_2,I_2)\agr(\varphi_2,i_2)=(\psi_2',I_2')$,\\
then following postulates hold:\\
$(P1)$ $\psi'$ implies $\varphi$.\\
$(P2)$ If $\psi\wedge\varphi$ is satisfiable, then $\psi'\equiv \psi\wedge\varphi$.\\
$(P3)$ If $\varphi$ is satisfiable, then $\psi'$ is also satisfiable.\\
$(P4)$ If $\psi\equiv\psi_2$ and $\varphi\equiv\varphi_2$ then $\psi'\equiv\psi_2'$.\\
$(P5)$ If $\psi\equiv\psi_2$ and $\varphi_2\equiv \varphi\wedge\varphi'$ then $\psi'\wedge\varphi'$ implies $\psi_2'$.\\
$(P6)$ If $\psi\equiv\psi_2$, $\varphi_2\equiv \varphi\wedge \varphi'$, and $\psi'\wedge\varphi'$ is satisfiable,\\
\hspace*{2cm}then $\psi_2'$ implies $\psi'\wedge\varphi'$.\\
$(P7)$ $(\psi',I')$ is coherent.\\
$(P8)$ If $(\psi',\{i\})$ is coherent, then $i\in I'$.\\
$(P9)$ If $(\psi',I\cup\{i\})$ is coherent, then $I\cup\{i\}\subseteq I'$.\\
$(P10)$ $I'\subseteq I\cup \{i\}$.\\
$(P11)$ If $I=I_2$, $i=i_2$, and $\psi'\equiv\psi_2'$, then $I'=I_2'$.\\
$(P12)$ For all $I''$ with $I'\subset I''\subseteq I\cup\{i\}$:$(\psi',I'')$ is not coherent.
\end{definition}

Postulates (P1)-(P6) are simply the KM postulates in our setting, which are equivalent to the AGM postulates~\cite{Katsuno1991}. They also state that the revision of strong beliefs does not depend on the intentions. Postulates (P7)-(P10) also appear in IPS. Postulate (P7) states that the outcome of a revision should be coherent. Postulate (P8) states that the new intention $i$ take precedence over all other current intentions; if possible, it should be added, even if all current intentions have to be discarded. Postulate (P9) and (P10) together state that if it is possible to simply add the intention, then this is the only change that is made. Postulate (P11) states that if we revise with the same $i$ but with a different belief, and we end up with the same belief in both cases, then we also end up with the same intentions. Finally, (P12) states that we do not discard intentions unnecessarily. This last postulate is comparable to the \emph{parsimony requirement} introduced by Grant \etal{}~\shortcite{grant2010}.

\paragraph{Representation Theorem} We next characterize all revision schemes satisfying (P1)-(P12) in terms of minimal change with respect to an ordering among interpretations and a selection function accommodating new intentions while restoring coherence. We bound models of strong beliefs up to $t$, which means that all the paths in the model are ``cut off'' at $t$. This ensures finitely many non-equivalent formulas for some belief set $B$. A \emph{$t$-bounded model} $\restm{} = (\restt,\restp)$ is a model containing a tree $T$ in which all paths, including $\pi$, have length $t$. Strong beliefs are about possibility and necessity, and they are independent of a specific path. Therefore, if a single path in a tree is a model of a strong belief, then all paths in this tree are models of this strong belief. Formally, a set of models of a strong belief $\strbel$ satisfies the following condition: $$\text{If }(T,\pi)\in M_{SB}\text{, then }(T,\pi')\in M_{SB}\text{ for all }\pi'\in T.$$ A set of \emph{t-bounded models of a strong belief} $\reststr$  contains only $t$ restricted models of a strong belief. We write $\reststrs$ to denote the set of all sets of $t$-bounded models of strong beliefs. We now show that we can represent a set of models of strong beliefs by a single formula.

\begin{lemma}
Let $Ext(\reststr)$ be the set of all possible extensions of a set of bounded model of strong beliefs $\reststr$ to models, i.e. $Ext(\reststr) = \{m\in\mods\mid \restm\in\reststr{}\}$. Given a set of $t$-bounded models of strong beliefs $\reststr$, there exists a strong belief formula $form(\reststr)$ such that $Mod(form(\reststr))=Ext(\reststr$).
\end{lemma}

\begin{corollary}
\label{cor:strongbeliefformula}
Given a $t$-bounded strong belief set $B$, there exists a formula $\psi$ such that $B=\{\varphi\mid\psi\vdash\varphi\}$.
\end{corollary}

\begin{proof}[Proof Sketch]
For a given belief set $B$, we can show that there exists a set of $t$-bounded models of a strong belief $\reststr$ s.t. $Ext(\reststr)=Mod(B)$.  If $\psi=form(\reststr)$, then $Mod(\psi)=Mod(B)$, and by the completeness theorem, $B=Cl(\psi)$.
\end{proof}

Given an intention database $I$, we define a selection function $\sel$ that tries to accommodate a new intention based on strong beliefs. The selection function specifies preferences on which intention an agent would like to keep in the presence of the new beliefs. 

\begin{definition}[Selection Function]
Given an intention database $I$, a \emph{selection function} $\sel:\strbels\times \ib \rightarrow \idb$ maps a set of models of a strong belief and an intention to an updated intention database---all bounded up to $t$--- such that if $\sel(\restM,\{i\})=I'$, then:
\begin{enumerate}
\item $(\restM,I')$ is coherent.
\item If $(\restM,\{i\})$ is coherent, then $i\in I'$.
\item If $(\restM,I\cup\{i\})$ is coherent, then $I\cup\{i\}\subseteq I'$.
\item $I'\subseteq I\cup\{i\}$.
\item For all $I''$ with $I'\subset I''\subseteq I\cup\{i\}$:$(\restM,I'')$ is not coherent.
\end{enumerate}
\end{definition}

The five conditions on the selection function are in direct correspondence with postulates (P7)-(P10), (P12) of the agent revision function $*_t$. Note that postulate (P11) doesn't have a corresponding condition in the definition above but is represented by the fact that the selection function takes the revised beliefs as input. That is, intention revision occurs after belief revision.

KM define a faithful assignment from a belief formula to a pre-order over models. Since we are also considering intentions, we extend this definition such that it also maps intentions databases to selection functions.

\begin{definition}[Faithful assignment]
A \emph{faithful assignment} is a function that assigns to each strong belief formula $\psi\in \sbel\rest{}$ a total pre-order $\lept$ over $\mods{}$ and to each intention database $I\in \mathbb{D}\rest{}$ a selection function $\sel$ and satisfies the following conditions:
\begin{enumerate}
\item If $m_1,m_2\in Mod(\psi)$, then $m_1\lept m_2$ and $m_2\lept m_1$.
\item If $m_1\in Mod(\psi)$ and $m_2\not\in Mod(\psi)$, then $m_1< m_2$.
\item If $\psi\equiv\phi$, then $\lept = \le_\phi^t$.
\item If $\restt=T_2\rest{}$, then $(T,\pi)\lept (T_2,\pi_2)$ and $(T_2,\pi_2)\lept (T,\pi)$.
\end{enumerate}
\end{definition}

Conditions 1 to 3 on the faithful assignment are the same as those of KM. Condition 4 ensures that we do not distinguish between models in the total pre-order $\lept$ whose trees are the same up to time $t$. This is essentially what is represented in the revision function by bounding the all input of the revision function $*_t$ up to $t$. Moreover, $\lept$ does not distinguish between models obtained by selecting two different paths from the same tree. This corresponds to the fact that we are using strong belief formulas in the revision, which do not distinguish between different paths in the same tree as well.

\begin{theorem}[Representation Theorem]
An agent revision operator $*_t$ satisfies postulates (P1)-(P12) iff there exists a faithful assignment that maps each $\psi$ to a total pre-order $\lept$ and each $I$ to a selection function $\sel$ such that if $(\psi,I)*_t(\varphi,i)=(\psi',I')$, then:
\begin{enumerate}
\item $Mod(\psi')=\min(Mod(\varphi),\lept)$
\item $I' = \sel(Mod(\psi'),i)$
\end{enumerate}
\end{theorem}

\begin{proof}[Proof Sketch]
We only sketch the proof of $``\Rightarrow'':$ Suppose that some agent revision operator $*_t$ satisfies postulates (P1)-(P12). Given models $m_1$ and $m_2$, let $(\psi,\emptyset)*_t(form(m_1\rest{})\vee form(m_2\rest),\epsilon)=(\psi',\emptyset)$. We define $\lept$ by $m_1\lept m_2$ iff $m_1\models\psi$ or $m_1\models\psi'$. We also define $\sel$ by $\sel(\restM_{SB},i)=I'$, where $(form(\restM_{SB}),I)*_t (\top,i)=(\psi_2,I')$ (note that $\psi_2\equiv form(\restM_{SB}$)). 

Let us prove condition 4 of Definition 9. For $m_1=(T,\pi)$ and $m_2=(T_2,\pi_2)$, let $\psi'$ be as above. Since $\psi,\psi'\in \sbel\rest{}$ and $\restt=T_2\rest$, we have $m_1\models\psi$ iff $m_2\models\psi$ and $m_1\models\psi'$ iff $m_2\models\psi'$, so $m_1\lept m_2$ and $m_2\lept m_1$. 

Following KM, one can show that conditions 1 to 3 from Definition 9 hold, and furthermore that $Mod(\psi')=\min(Mod(\varphi),\lept)$. We now prove $I' = \sel(Mod(\psi'),i)$. By our definition of $\sel$ we have that $(\psi',I)*_t (\top,i)=(\psi_2,\sel(Mod(\psi'),i))$ (recall that $\psi'\equiv \psi_2$). Since $(\psi,I)*_t(\varphi,i)=(\psi',I')$, by (P11) we obtain that $I' = \sel(Mod(\psi'),i)$. Using postulate (P7)-(P10) and (P12) we can prove that $\sel$ is a selection function.
\end{proof}

Finally, it turns out to be straightforward to formulate the DP postulates for iterated revision in our framework for the strong beliefs and to prove their representation theorem. Due to space constraints we have omitted the results, but they can be found in a separate technical report.\footnote{\\ http://www.dropbox.com/s/798g6ki0aqt3gds/ijcai2016proofs.pdf?dl=1}

\section{Related Work}
Grant \etal{}~\shortcite{grant2010} develop AGM-style postulates for belief, intention, and goal revision. They provide a detailed analysis and propose different reconsideration strategies, but restrict themselves to a syntactic analysis. Much effort in combining AGM revision with action logics (e.g., the Event Calculus~\cite{Mueller2010}, Temporal Action Logics~\cite{Kvarnstrom2005}, extensions to the Fluent Calculus~\cite{Thielscher2001}, and extensions to the Situation Calculus (see~\cite[Ch.2]{Patkos2010} for an overview)) concentrates on extending these action theories to incorporate \emph{sensing} or \emph{knowledge-producing actions}. Shapiro \etal{}~\shortcite{Shapiro2011} extend the Situation Calculus to reason about beliefs rather than knowledge by introducing a modality $B$ and shows that both the AGM postulates and the DP postulates are satisfied in this framework. A similar approach concerning the Fluent Calculus has been formalized by Jin and Thielscher~\shortcite{Jin2004}, and is further developed by Scherl~\shortcite{Scherl2005} and Scherl and Levesque~\shortcite{Scherl2003} by taking into account the frame problem as well. However, none of these approaches prove representations theorems linking revision to a total pre-order on models. Baral and Zhang~\shortcite{Baral2005} model belief updates on the basis of semantics of modal logic S5 and show that their knowledge update operator satisfies all the KM postulates.  Bonanno~\shortcite{Bonanno2007} combines temporal logic with AGM belief revision by extending a temporal logic with a belief operator and an information operator. Both these approaches do not take action or time into account and do not prove representation theorems. The concept of strong beliefs has been discussed extensively in the literature, for instance in the story of \emph{Little Nell}~\cite{Mcdermott1982} or a paradox found in \emph{knowledge-based programs}~\cite{Fagin1995} (see van der Hoek \emph{et al.}~\shortcite{vanderHoek2003} for a detailed discussion).

\section{Conclusion}

We develop a logical theory for reasoning about temporal beliefs and intentions based on Shoham's database perspective. We propose postulates for revision of strong beliefs and intentions, and prove a representation theorem relating the postulates to our formal model. \icard{} prove a comparable representation theorem, but we show in this paper that it does not hold in general. In their proof, they use a canonical model construction, while we proof our representation theorem using standard techniques from belief revision. It remains an open problem whether a canonical model construction is possible when proving the representation theorem of \icard

For future work, we aim to extend our formalism with goals, which seems a natural extension in order to allow the agent to, for instance, replace intentions instead of merely discarding them. This paves the road to develop a richer notion of intentions, such as that ``intentions normally pose problems for the agents; the agent needs to determine a way of achieving them''~\cite{Cohen1990}. Interestingly, adding goals to the formalism blurs the distinction between planner and databases. If the databases take over part of the planning, then well-known problems such as the frame problem become more stringent. Existing action logics (e.g., the Event Calculus or the Fluent Calculus) and database approaches (e.g., TMMS~\cite{Dean1987}) have dealt with these problems in detail, so comparing and possibly enriching them with our formalism seems both useful and relevant future work.

\section{Acknowledgments}
We thank Leon van der Torre and Eric Pacuit for useful comments. Marc van Zee and Dragan Doder are both funded by the National Research Fund (FNR), Luxembourg, respectively by the  RationalArchitecture and the PRIMAT project.

\bibliographystyle{named}
\bibliography{refs}

\begin{thebibliography}{}

\bibitem[\protect\citeauthoryear{Alchourron \bgroup \em et al.\egroup
  }{1985}]{AGM1985}
Carlos~E. Alchourron, Peter G\"{a}rdenfors, and David Makinson.
\newblock On the logic of theory change: Partial meet contraction and revision
  functions.
\newblock {\em Journal of Symbolic Logic}, 50(2):510--530, 06 1985.

\bibitem[\protect\citeauthoryear{Baral and Zhang}{2005}]{Baral2005}
Chitta Baral and Yan Zhang.
\newblock Knowledge updates: Semantics and complexity issues.
\newblock {\em Artificial Intelligence}, 164(1):209--243, 2005.

\bibitem[\protect\citeauthoryear{Bonanno}{2007}]{Bonanno2007}
Giacomo Bonanno.
\newblock Axiomatic characterization of the {AGM} theory of belief revision in
  a temporal logic.
\newblock {\em Artificial Intelligence}, 171(2):144--160, 2007.

\bibitem[\protect\citeauthoryear{Bratman}{1987}]{Bratman1987}
Michael~E. Bratman.
\newblock {\em Intention, plans, and practical reason}.
\newblock Harvard University Press, Cambridge, MA, 1987.

\bibitem[\protect\citeauthoryear{Cohen and Levesque}{1990}]{Cohen1990}
Philip~R Cohen and Hector~J Levesque.
\newblock Intention is choice with commitment.
\newblock {\em {Artificial Intelligence}}, 42(2-3):213--261, 1990.

\bibitem[\protect\citeauthoryear{Cohen and Levesque}{1991}]{Cohen1991}
Philip~R. Cohen and Hector~J. Levesque.
\newblock Teamwork.
\newblock {\em No\^u{}s}, 25(4):487--512, 1991.

\bibitem[\protect\citeauthoryear{Darwiche and Pearl}{1997}]{Darwiche1997}
Adnan Darwiche and Judea Pearl.
\newblock On the logic of iterated belief revision.
\newblock {\em Artificial Intelligence}, 89(1–2):1 -- 29, 1997.

\bibitem[\protect\citeauthoryear{Dean and McDermott}{1987}]{Dean1987}
Thomas~L Dean and Drew~V McDermott.
\newblock Temporal data base management.
\newblock {\em Artificial Intelligence}, 32(1):1--55, 1987.

\bibitem[\protect\citeauthoryear{Ditmarsch \bgroup \em et al.\egroup
  }{2011}]{Ditmarsch2011}
Hans Ditmarsch, Tiago Lima, and Emiliano Lorini.
\newblock {\em Intention Change via Local Assignments}, pages 136--151.
\newblock Springer Berlin Heidelberg, Berlin, Heidelberg, 2011.

\bibitem[\protect\citeauthoryear{Fagin \bgroup \em et al.\egroup
  }{1995}]{Fagin1995}
Ronald Fagin, Joseph~Y. Halpern, Yoram Moses, and Moshe~Y. Vardi.
\newblock {\em Reasoning about Knowledge}.
\newblock MIT Press, 1995.

\bibitem[\protect\citeauthoryear{Grant \bgroup \em et al.\egroup
  }{2010}]{grant2010}
John Grant, Sarit Kraus, Donald Perlis, and Michael Wooldridge.
\newblock {Postulates for revising BDI structures}.
\newblock {\em Synthese}, 175(1):39--62, 2010.

\bibitem[\protect\citeauthoryear{Icard \bgroup \em et al.\egroup
  }{2010}]{Icard2010}
Thomas Icard, Eric Pacuit, and Yoav Shoham.
\newblock {Joint revision of belief and intention}.
\newblock {\em Proc. of the 12th International Conference on Knowledge
  Representation}, pages 572--574, 2010.

\bibitem[\protect\citeauthoryear{Jin and Thielscher}{2004}]{Jin2004}
Yi~Jin and Michael Thielscher.
\newblock Representing beliefs in the fluent calculus.
\newblock In {\em ECAI}, pages 823--827. IOS Press, 2004.

\bibitem[\protect\citeauthoryear{Katsuno and Mendelzon}{1991}]{Katsuno1991}
Hirofumi Katsuno and Alberto~O. Mendelzon.
\newblock Propositional knowledge base revision and minimal change.
\newblock {\em Artificial Intelligence}, 52(3):263--294, dec 1991.

\bibitem[\protect\citeauthoryear{Kvarnstr\"om}{2005}]{Kvarnstrom2005}
Jonas Kvarnstr\"om.
\newblock {\em {TALplanner and other extensions to Temporal Action Logic}}.
\newblock PhD thesis, Link\"opings universitet, 2005.

\bibitem[\protect\citeauthoryear{Lorini and Herzig}{2008}]{Lorini2008}
Emiliano Lorini and Andreas Herzig.
\newblock A logic of intention and attempt.
\newblock {\em Synthese}, 163(1):45--77, 2008.

\bibitem[\protect\citeauthoryear{Lorini \bgroup \em et al.\egroup
  }{2009}]{Lorini2009}
Emiliano Lorini, Mehdi Dastani, Hans~P. van Ditmarsch, Andreas Herzig, and
  John-Jules~Ch. Meyer.
\newblock Intentions and assignments.
\newblock In {\em LORI}, volume 5834 of {\em Lecture Notes in Computer
  Science}, pages 198--211. Springer, 2009.

\bibitem[\protect\citeauthoryear{McDermott}{1982}]{Mcdermott1982}
Drew McDermott.
\newblock A temporal logic for reasoning about processes and plans.
\newblock {\em Cognitive science}, 6(2):101--155, 1982.

\bibitem[\protect\citeauthoryear{Mueller}{2010}]{Mueller2010}
Erik~T Mueller.
\newblock {\em Commonsense reasoning}.
\newblock Morgan Kaufmann, 2010.

\bibitem[\protect\citeauthoryear{Patkos}{2010}]{Patkos2010}
Theodore Patkos.
\newblock {\em A formal theory for reasoning about action, knowledge and time}.
\newblock PhD thesis, University of Crete-Heraklion, 2010.

\bibitem[\protect\citeauthoryear{Reynolds}{2002}]{Reynolds2002}
M.~Reynolds.
\newblock An axiomatization of full computation tree logic.
\newblock {\em Journal of Symbolic Logic}, 66(3):1011--1057, 2002.

\bibitem[\protect\citeauthoryear{Scherl and Levesque}{2003}]{Scherl2003}
Richard~B Scherl and Hector~J Levesque.
\newblock Knowledge, action, and the frame problem.
\newblock {\em Artificial Intelligence}, 144(1):1--39, 2003.

\bibitem[\protect\citeauthoryear{Scherl}{2005}]{Scherl2005}
Richard~B Scherl.
\newblock Action, belief change and the frame problem: A fluent calculus
  approach.
\newblock In {\em Proceedings of the Sixth workshop on Nonmonotonic Reasoning,
  Action, and Change at IJCAI}, 2005.

\bibitem[\protect\citeauthoryear{Shapiro \bgroup \em et al.\egroup
  }{2011}]{Shapiro2011}
Steven Shapiro, Maurice Pagnucco, Yves Lespérance, and Hector~J. Levesque.
\newblock Iterated belief change in the situation calculus.
\newblock {\em Artificial Intelligence}, 175(1):165--192, 2011.

\bibitem[\protect\citeauthoryear{Shoham}{2009}]{Shoham2009}
Yoav Shoham.
\newblock Logical theories of intention and the database perspective.
\newblock {\em Journal of Philosophical Logic}, 38(6):633--647, 2009.

\bibitem[\protect\citeauthoryear{Shoham}{2016}]{Shoham2016}
Yoav Shoham.
\newblock Why knowledge representation matters.
\newblock {\em Commun. ACM}, 59(1):47--49, January 2016.

\bibitem[\protect\citeauthoryear{Thielscher}{2001}]{Thielscher2001}
Michael Thielscher.
\newblock The concurrent, continuous fluent calculus.
\newblock {\em Studia Logica}, 67(3):315--331, 2001.

\bibitem[\protect\citeauthoryear{Van~der Hoek and
  Wooldridge}{2003}]{vanderHoek2003}
Wiebe Van~der Hoek and Michael Wooldridge.
\newblock Towards a logic of rational agency.
\newblock {\em Logic Journal of IGPL}, 11(2):135--159, 2003.

\bibitem[\protect\citeauthoryear{van~der Hoek \bgroup \em et al.\egroup
  }{2007}]{vanderHoek2007}
Wiebe van~der Hoek, Wojciech Jamroga, and Michael Wooldridge.
\newblock {Towards a theory of intention revision}.
\newblock {\em Synthese}, 155(2):265--290, 2007.

\bibitem[\protect\citeauthoryear{van Zee \bgroup \em et al.\egroup
  }{2015a}]{vanZee2015}
Marc van Zee, Mehdi Dastani, Dragan Doder, and Leendert van~der Torre.
\newblock {Consistency Conditions for Beliefs and Intentions}.
\newblock In {\em Twelfth International Symposium on Logical Formalizations of
  Commonsense Reasoning}, 2015.

\bibitem[\protect\citeauthoryear{van Zee \bgroup \em et al.\egroup
  }{2015b}]{vanZee2015b}
Marc van Zee, Dragan Doder, Mehdi Dastani, and Leendert van~der Torre.
\newblock {AGM Revision of Beliefs about Action and Time}.
\newblock In {\em Proceedings of the International Joint Conference on
  Artificial Intelligence}, 2015.

\end{thebibliography}

\end{document}